\def\year{2021}\relax
\documentclass[letterpaper]{article} % DO NOT CHANGE THIS
\usepackage{aaai21}  % DO NOT CHANGE THIS
\usepackage{times}  % DO NOT CHANGE THIS
\usepackage{helvet} % DO NOT CHANGE THIS
\usepackage{courier}  % DO NOT CHANGE THIS
\usepackage[hyphens]{url}  % DO NOT CHANGE THIS
\usepackage{graphicx} % DO NOT CHANGE THIS
\usepackage{natbib}  % DO NOT CHANGE THIS AND DO NOT ADD ANY OPTIONS TO IT
\usepackage{caption} % DO NOT CHANGE THIS AND DO NOT ADD ANY OPTIONS TO IT
\usepackage{amsthm}
\usepackage{amsmath}
\usepackage{amssymb}
\usepackage{booktabs}
\RequirePackage{algorithm}
\RequirePackage{algorithmic}

\newtheorem{theorem}{Theorem}
\newtheorem{lemma}{Lemma}

\title{Quantifying Membership Inference Vulnerability via Generalization Gap and Other Model Metrics\thanks{This research was funded by the National Geospatial-Intelligence Agency and was approved for public release as document \#20-571.}}
\author {
        Jason W. Bentley\textsuperscript{\rm 1}~~   
        Daniel Gibney\textsuperscript{\rm 1}~~   
        Gary Hoppenworth\textsuperscript{\rm 1}~~   
        Sumit Kumar Jha\textsuperscript{\rm 2}\\
        }
\affiliations {
    \textsuperscript{\rm 1} Department of Computer Science, University of Central Florida \\
    \textsuperscript{\rm 2} Department of Computer Science, University of Texas San Antonio\\
    jason.w.bentley@knights.ucf.edu, daniel.j.gibney@gmail.com, gary.hoppenworth@gmail.com, Sumit.Jha@utsa.edu
}
\begin{document}

\maketitle

\begin{abstract}
We demonstrate how a target model's generalization gap leads directly to an effective deterministic black box membership inference attack (MIA). This provides an upper bound on how secure a model can be to MIA based on a simple metric. 
Moreover, this attack is shown to be optimal in the expected sense given access to only certain likely obtainable metrics regarding the network's training and performance. Experimentally, this attack is shown to be comparable in accuracy to state-of-art MIAs in many cases. %We then show how to utilize the concept of shadow models to gain more information about a target model, leading to even more effective attacks. 
\end{abstract}

Personal data has become a commodity to be bought, sold, and analyzed for making predictions into how people will behave. One primary concern is protecting sensitive personal information against extraction from analytic tools and from aggregate data.
This personal information can include birthdays, medical records, etc., that individuals may wish to not be publicly accessible. A recent trend in utilizing this personal data is the development of application programming interfaces (APIs) that expose a model trained on this data. These APIs are made available to companies and individuals, thus giving them some of the utility of the data, without granting them access to the data directly. Typically, users may query the model on particular inputs and receive the model's output, but cannot access the model directly. Sensitive information is used to train these models, however this sensitive information is not intended to be available to the APIs' users. Unfortunately, this may not always be the case, as APIs utilizing overfit models have been shown to be vulnerable to attackers attempting to extract this information.
This work gives an upper bound on how effective such an attack can be when the attacker has limited knowledge of the model's training and performance. 
%These models are most often created by first choosing a model architecture with unfitted parameters and then applying algorithms to fit the model's parameters by utilizing an enormous data set, the likes of which were previously unavailable(i.e., machine learning). 
%These APIs are readily available and have proven helpful in many applications. 
%Yet, on the other hand, as individuals we still value our privacy and generally wish for our personal information such as birthdays, medical records, etc.,  to not be publicly accessible. 
%This creates two contradictory forces, one that seeks to extract and utilize all information within each data point representing an individual in an effort to make meaningful predictions, and one which aims to obfuscate and protect this information. 
%The question of whether the utility of a model fitted with such data can be maintained while ensuring personal information remains private is a vexing one.

% To formalize this problem mathematically, we first note that set membership represents a simple, yet general enough concept to express many of the properties that a data point may possess. That is,  knowing whether a given data point is a member of a specific set is likely expressive enough to capture whatever property one wishes to specify. In this context, membership inference arises when one cannot determine this membership directly from the data point but rather has to infer this by using the aforementioned model. We follow others by making 
% the set
% of interest the data points that are used to fit the model's parameters. Specifically, 
In a membership inference attack, we consider the attacker as having access to all input-label data points $z = (x,y)$, the target model's output on input $x$, which we denote as $M(x)$, and some auxiliary information regarding the target model's performance (specified later). The attacker then aims to determine whether $z$ belongs to the training set used to fit target model $M$'s parameters. More details are provided in the Attack Setting section. We note that the attack outlined in this paper is still effective in the more general context of determining membership within an arbitrary set. It could be employed for an arbitrary set as long as the attacker has similar auxiliary information for the set of interest. 
%accuracy of the model within and outside the set, and the relative proportion of data points that lie within the set as well. 
%But, since the information training and testing accuracy of a model is information which seems more likely to be readily available, membership inference for the training set is the most natural candidate.

For membership inference, this paper provides a simple and theoretically optimal attack strategy (under certain conditions). 
Although our attack requires auxiliary information regarding the target model's performance, the information needed by the attacker could be obtained either through knowledge of how the model is created or by `modeling the model' through the creation of \textit{shadow models}, first introduced in \cite{DBLP:conf/sp/ShokriSSS17}. Thanks to the optimality of our method, we obtain a bound on the effectiveness of any attack,
provided the attacker does not possess additional information beyond what we assume.
%provided the attacker has access to the same information. 
We then compare our strategy to  state-of-the-art membership inference attacks that have access to potentially much more information than our attack. Our results help to shed light on which information is likely being used in these membership inference attacks, and thus reveal which techniques, if any, might be useful in shielding against these attacks.

\subsection{Related Work}
\label{sec:related_work}
The topic of membership inference attacks has been of growing interest within the past decade. \emph{Differential privacy} has been one of the major tools used to try to formalize the notion of being secure against membership inference. Although the topic of Differential Privacy is quite old and dates back many decades, the formalized concept of $\varepsilon$-Differential Privacy was introduced by Dwork et al., 2006 in \cite{DBLP:conf/tcc/DworkMNS06} and lead to an abundance of related research 
%in a wide range of topics from combinatorial optimimization , 
including its application to machine learning~\cite{DBLP:journals/corr/JiLE14, DBLP:conf/ccs/AbadiCGMMT016, DBLP:journals/ml/PhanWD17, DBLP:conf/sp/Yu0PGT19}. A typical idea behind the application of differential privacy is to apply Gaussian mechanisms to the training procedure (or the objective function~\cite{DBLP:journals/jpc/WangKL19} used to train  the network). This can be mathematically proven to provide some level of security (see \cite{DBLP:conf/icml/BalleW18} for a work devoted specifically to Gaussian mechanisms). In practice, it also leads to a loss in performance and so trade-offs have to be made between privacy, training time, and accuracy. Importantly, our approach only considers the generalization gaps between the training and testing accuracy on various subsets of the data. It follows that \emph{whatever protection from our attack is afforded through differential privacy techniques is only the result of affecting the accuracy of the model}.

This is not the first result along the lines of demonstrating that security comes at a cost of affecting usefulness. Another metric of privacy is \emph{k-Anonymity}~\cite{DBLP:journals/ijufks/Sweene02}. Under this metric, one is concerned with whether a record is indistinguishable from at least $k$ other data points in the data set. It was shown in \cite{DBLP:conf/vldb/Aggarwal05} that k-anonymity requires a significant amount of information loss, particularly for higher dimensional data. Further, it was shown in \cite{DBLP:conf/kdd/BrickellS08} that the removal of member-identifying attributes is typically more effective than most k-anonymity methods.

An approach to membership inference that is working in the opposite direction of differential privacy or k-anonymity is to design effective attacks. Early approaches to designing these attacks can be seen in \cite{homer2008resolving, DBLP:conf/ccs/0001BHM16, DBLP:conf/focs/DworkSSUV15}. These are largely based on using some form of a distance measure to determine how likely it is that a particular data point is in a subset of the data. One of the most notable recent efforts in this direction is given by Shokri et al. in \cite{DBLP:conf/sp/ShokriSSS17} whose shadow model based attacks are actually used for comparison with the results in this paper. A work by Rahman et al., empirically evaluates how well differential privacy techniques do against these attacks~\cite{DBLP:journals/tdp/RahmanRLM18}.
Their research concludes that to ensure privacy, a significant amount of utility of the model must be lost.

\subsection{Results}
The main result of this work is the explicit and simple attack strategy first given in Algorithm \ref{alg:take-the-typical}, and then expanded in the Categorical Bayesian Take-the-Typical Attack section. As formalized in Theorem \ref{thm:optimal}, we will see that this strategy is provably optimal given only access to a model's training accuracy, testing accuracy, and the proportion of data used for the training set within each subset of the data formed by the partition. We emphasize that this strategy is also easy to implement. In the simplest case where the partition consists of only one part, the expected accuracy, recall, and precision of the method have explicit formulas given in the Metrics section. %Experimentally, when comparing the state-of-the-art attacks of \cite{DBLP:conf/sp/ShokriSSS17} to the theoretical values we can compute, the results are very often nearly the same. 

We note that the simplicity of our approach does not come at the cost of poor performance. \textit{Experimentally, we find that the accuracy of our membership inference attack is very often comparable with that of the state-of-the-art attacks of \cite{DBLP:conf/sp/ShokriSSS17}, despite our attack not requiring a trained attack model}.

%This indicates that their attack model is often extracting little more useful information than the training and testing accuracy. 

\section{Background}

\subsection{Attack Setting}
\label{sec:setting}
Each data point $z$ in our data set $D$ consists of an input $x$ and a true label $y$. We label these as $z = (x,y)$. 
We assume throughout that the attacker has black box access to a target model $M$ so that given a data point $z = (x,y)$ they can obtain the output of the model, $M(x)$, which we initially take as simply being a predicted label. In the Categorical Bayesian Take-the-Typical Attack section, $M(x)$ is generalized to be a probability distribution over all possible labels. For every data point $x$, we assume the attacker has access to the true label $y$. Additionally, we assume the attacker has access to the training and testing accuracy for a  subset of the data, which we will see is not an unreasonable assumption in previously published attacks as well. We lastly assume that the attacker has knowledge of the proportion of overall data is being used for training. Considering that these proportions are often chosen by some commonly known rule of thumb, an attacker could realistically know this as well.

\subsection{Assessment and Previous MIAs}
The assessment of our attack is based on several commonly used statistical metrics. These include accuracy, precision, and recall of the attack. In the simplest case where the data is not further partitioned, each of these have simple formulas that allow for their exact computation. Because these values are compared with the attacks of \cite{DBLP:conf/sp/ShokriSSS17} we briefly outline how those attacks work here.

\textbf{Description of SMA:} The attack of Shokri et al., has two main parts. The first is the creation of shadow models and the second is the creation of an attack model. Within the shadow model creation phase, a collection of models are created which closely imitate the behavior of the target model. For these shadow models one then knows whether or not a particular data point was within its training set. Hence, we can use both the output of the shadow model and a label of `in training set' or 'not in training set' to train an attack model. When given an output of the target model, the attack model can attempt to infer whether the target model's input was in its training set. Although, many of the details of this attack are omitted, this conveys the essence of the attack. For notational convenience, we will call this the Shadow Model Attack which will henceforth abbreviate as SMA.

For assessing the accuracy of the SMA, Shokri et al. set the number of elements in the training set and the number of elements not in the training set to be equal. This is a reasonable assumption, and we adopt it for all experimental evaluations of our attack.
%Indeed, over an unbounded domain of data points where the points are pulled at random it would be hard to have a higher accuracy than an attack which always guesses the point is not in the training set.

\section{Bayesian Take-the-Typical Attack (BTTA)}
\label{sec:BTTA}
%\subsection{Attack Description}

This section outlines the most basic version of the attack where the data set is not partitioned (before considering the partitioning into training and testing data). We first introduce some notation.
\begin{itemize}
    \item The set of inputs is denoted by $X$ and the set of labels is denoted by $Y$.
    \item The data set is $D \subseteq X \times Y$ where an element from $X$ can be used at most once. The set $D$ is partitioned in to the training set $D^{train}$ and the testing set $D^{test}$.
    \item The target model is denoted by $M$, where $M$ is a function $M: X \mapsto Y$.
    \item The probability that a data point $z = (x,y) \in D$ is drawn from the training set $D^{train}$ is given by $q$.
    \item Let $A$ denote the set of data points  $M$ correctly classifies, i.e., 
    $A = \{z=(x,y) \in D :  M(x) = y \}$.
    \item The accuracy of the model $M$ on the training set is notated as $p_0$, i.e.,
    $p_0 
    %= P(M(x) = y \mid z = (x,y) \in D^{train}) 
    = P(z \in A \mid z \in D^{train})$.
    \item The accuracy of the model $M$ on the testing set is notated as $p_1$,  i.e.,
    $p_1 
    = P(z \in A \mid z \in D^{test})$.
\end{itemize}
We will always assume that the model $M$ is more accurate on $D^{train}$ than on $D^{test}$ so that $p_0 \geq p_1$. We are now ready to present what we call the Bayesian Take-the-Typical Attack (BTTA)\footnote{The attack is written to iterate over all elements in $D$, but of course could be applied to only a particular element.}. We name it this in contrast to a even simpler attack we will call Take-the-Typical. In the Take-the-Typical Attack if $q \geq 1/2$ the attacker will always report that the data point is in the training set and if $q < 1/2$ the attacker will always report the data point is not in the training set. The Take-the-Typical attack ignores the models behaviour on a input whereas the BTTA exploits it. The derivation of the BTTA is presented as a proof to Lemma \ref{lem:BTTA_derivation}, the statement of which highlights the attack's most salient feature.
\begin{algorithm}[tb]
   \caption{Bayesian Take-the-Typical Attack (BTTA)}
   \label{alg:take-the-typical}
\begin{algorithmic}
   \STATE {\bfseries Input:} Model $M$, $q$, $p_0$, $p_1$, data set $D$
   \FOR{$z = (x,y) \in D$}
   \IF{$M(x) = y$}  
   \IF{$q p_0 \geq (1-q)p_1$}
   \STATE report $x \in D^{train}$
   \ELSE
   \STATE report $x \notin D^{train}$
   \ENDIF
   \ELSIF{$M(x) \neq y$}
   \IF{$q(1-p_0) \geq (1-q)(1-p_1)$}
   \STATE report $x \in D^{train}$
   \ELSE
   \STATE report $x \notin D^{train}$
   \ENDIF
   \ENDIF
   \ENDFOR
\end{algorithmic}
\end{algorithm}

\begin{lemma}
\label{lem:BTTA_derivation}
The Bayesian Take-the-Typical Attack reports $x \in D^{train}$ iff it is more probable $x \in D^{train}$.
\end{lemma}
\begin{proof}
By Bayes rule,
\begin{align*}
&P(z \in D^{train} \mid z \in A)\\
&= \frac{P(z \in A \mid z \in D^{train})P(z\in D^{train})}{ P(z\in A)}\\ &= \frac{qp_0}{qp_0 + (1-q)p_1}.
\end{align*}
Setting this greater or equal to $1/2$ we see that given the correct classification, i.e., that $z \in A$, it is more probable that $z \in D^{train}$ when $qp_0 \geq (1-q)p_1$. 

Similarly,
\begin{align*}
&P(z \in D^{train} \mid  z \notin A)\\
&= \frac{P(z \notin A \mid z \in D^{train})P( z \in D^{train})}{P(z \notin A)}\\ &= \frac{q(1-p_0)}{q(1-p_0) + (1-q)(1-p_1)}.
\end{align*}
Setting this greater or equal to $1/2$, we can see that given the misclassification, i.e., that $z \notin A$, it most probable $z \in D^{train}$ when $q(1-p_0) \geq (1-q)(1-p_1)$. 
\end{proof}

The next theorem is a corollary of the fact that the above attack strategy picks the most likely answer for any given data point. Assuming the information we are provided is correct and the data points we are pulling from are uniformly distributed, there is no advantage to updating any of our information and we have Theorem \ref{thm:optimal}. Any attack that made a choice which was more probable to be incorrect than correct could be improved in expectation by always taking the solution which is more probable to be correct.
\begin{theorem}
\label{thm:optimal}
Given only the access to the model, training and test accuracy, the Bayesian Take-the-Typical Attack is optimal with respect to accuracy.
\end{theorem}
Now that we have shown the optimality of the BTTA we next derive formulas for its accuracy, precision, and recall after which we can relate it to other attacks, namely Shokri et al's SMA.
\subsection{Metrics - Lower Bounds on MIA Vulnerability}
\label{sec:attack_prec_and_accuracy}

The conditional statements in Algorithm \ref{alg:take-the-typical} allow for the four cases to be processed differently. We define them now and will refer to them as Cases 1-4 throughout the paper.

{\bf Case 1.} $qp_0 \geq (1-q)p_1$ and $q(1-p_0) \geq (1-q)(1-p_1)$.

{\bf Case 2.} $qp_0 < (1-q)p_1$ and $q(1-p_0) < (1-q)(1-p_1)$.

{\bf Case 3.} $qp_0 \geq (1-q)p_1$ and $q(1-p_0) < (1-q)(1-p_1)$.

{ \bf Case 4.} $qp_0 < (1-q)p_1$ and $q(1-p_0) \geq (1-q)(1-p_1)$.

Thanks to the following lemma, we need henceforth only address Cases 1-3.
\begin{lemma}
Under the assumption testing accuracy is at least training accuracy, or $p_0 \geq p_1$, Case 4 will never occur.
\end{lemma}
\begin{proof}
If $q > 1-q$, then $qp_0 < (1-q)p_1$ implies $p_1 > p_0$, which is not possible. On the other hand, if $q \leq 1-q$, then $q(1-p_0) \geq (1-q)(1-p_1)$ implies $1-p_0 \geq 1-p_1$ which can only happen if $p_0$ and $p_1$ are equal and $q = 1/2$, but now the first inequality again fails to be satisfied. 
\end{proof}

Lemma \ref{lem:beats_TTA} demonstrates that it can be advantageous to use the training and testing accuracy within the attack. In fact there exists a range of values of $p_0$, $p_1$, and $q$ where $p_0$ and $p_1$ are critical in predicting membership, and where outside this range $p_0$ and $p_1$ are no longer of any use at all in the prediction of set membership. The prior of these is captured by Case 3 and the latter by Cases 1 and 2.

\begin{lemma}
\label{lem:beats_TTA}
The accuracy of Bayesian Take-the-Typical Attack is always better than or equal to $\max \{q, 1-q\}$, the accuracy of Take-the-Typical.
\end{lemma}

\begin{proof}

{\bf Case 1.} This coincides with Take-the-typical and always reports $z \in D^{train}$ with accuracy $q$. This only occurs when $q \geq 1-q$, showing that accuracy is at least $\max\{q,1-q\}$.

{\bf Case 2.} This coincides with Take-the-typical and always reports $z \notin D^{train}$ with accuracy $1-q$. This only occurs when $1-q > q$, showing that accuracy is at least $\max\{q,1-q\}$.

{\bf Case 3.} The Bayesian Take-the-typical will report $z \in D^{train}$ when $z \in A$ and will report $z \notin D^{train}$ when $z \notin A$. 

The accuracy is then be given by
\begin{align}
P(z \in D^{train} &\land z \in A) + P(z \notin D^{train} \land z \notin A) \nonumber\\ 
&= qp_0 + (1-q)(1-p_1) \label{eq:case3_acc}\\
&> qp_0 + q(1-p_0) = q. \nonumber
\end{align}

where we used $(1-q)(1-p_1) > q(1-p_0)$. Also note that since $qp_0 \geq (1-q)p_1$, we can say
\begin{align*}
P(z \in D^{train} &\land z \in A) + P(z \notin D^{train} \land z \notin A) \\
&=qp_0 + (1-q)(1-p_1)\\ 
&\ge (1-q)p_1 + (1-q)(1-p_1) = 1-q.
\end{align*}
Therefore, the accuracy is at least $\max\{q,1-q\}$.
\end{proof}

The BTTA has the following attack metrics.
\begin{itemize}
\item \textbf{Expected Accuracy:}
We again consider the Cases 1-3. Using the results derived in Lemma \ref{lem:beats_TTA} we can claim that in
\begin{itemize}
    \item Case 1: the expected accuracy is $q$; 
    \item Case 2: the expected accuracy is $1-q$;
    \item Case 3: since we report $z \in D^{train}$ iff $z \in A$ the expected accuracy is given by Equation \ref{eq:case3_acc}.
    \end{itemize}

\item \textbf{Expected Precision:}  
Abbreviating $D^{train}$ as $D^{tr}$, the expected precision of the BTTA is based on the ratio
$$
\tiny
\frac{P(\text{report } z \in D^{tr} \land z \in D^{tr})}{P(\text{report } z \in D^{tr} \land z \in D^{tr}) + P(\text{report } z \in D^{tr} \land z \notin D^{tr})}.
$$
Once again we have Cases 1-3, and in 
\begin{itemize}
    \item Case 1: $\frac{q}{q + (1-q)} = 1$;
    \item Case 2: the precision is not defined since we never report $z \in D^{train}$;
    \item Case 3: Since here we report $z \in D^{train}$ iff $M(x) = y$ the above expression is equal to
    \begin{align}
        &\frac{P(z \in A \land z \in D^{tr})}{P(z \in A \land z \in D^{tr}) + P(z \in A \land z \notin D^{tr})} \nonumber\\
        &= \frac{qp_0}{qp_0 + (1-q)p_1}. \label{eq:case3_prec}
    \end{align}
\end{itemize}
\item \textbf{Expected Recall:}
Lastly, the expected recall is based on the ratio
$$
\tiny
    \frac{P(\text{report } z \in D^{tr} \land z \in D^{tr})}{P(\text{report } z \in D^{tr} \land z \in D^{tr}) + P(\text{report } z \notin D^{tr} \land z \in D^{tr})}.
$$
We have in 
\begin{itemize}
    \item Case 1: the expected recall is $\frac{q}{q+0} =1$;
    \item Case 2: the expected recall is $\frac{0}{0+(1-q)1} = 0$;
    \item Case 3: the expected recall is
    \begin{align*}
        &\frac{P(z \in A \land z \in D^{train})}{P(z \in A \land z \in D^{train}) + P(z \notin A \land z \in D^{train})} \\
        &= \frac{qp_0}{qp_0 + q(1-p_0)}= p_0.
    \end{align*}
\end{itemize}
\end{itemize}
We focusing on the accuracy to prove Theorem \ref{thm:lower_bound}. The aim is to now show that regardless of which ever of the three cases may apply based on our values of $p_0$, $p_1$ and $q$, the accuracy of BTTA is always bound below by the maximum of the three case's expected accuracy.

\begin{theorem}
\label{thm:lower_bound}
Given access to a model with generalization gap $g = p_0 - p_1 \geq 0$ (training accuracy minus testing accuracy) and the ratio of training set to input domain $|A|/|\mathcal{D}| = q$,  there exists a membership inference attack with expected accuracy at least
\begin{align*}
    &\max\{q,1-q,qp_0 + (1-q)(1-p_1)\}\\
    &\geq \max\{q,1-q,\min\{q,1-q\}(1+g)\}
    \geq \frac{1}{2}
\end{align*}
Moreover, given only this information about the model no attack can have higher accuracy. 
\end{theorem}
\begin{proof}
In each Case $i \in \{1,2,3\}$ we will demonstrate that the lower bounds we derived for Case $i$ in Metrics Section is larger than the lower bounds given for the remaining Cases $\{1,2,3\}-\{i\}$.

For Case 1, $\max\{q,1-q\} = q$ and  $\min\{q,1-q\} = 1-q$. One can check that the only solution to the set of inequalities
\begin{align*}
    q  &< qp_0 + (1-q)(1-p_1) \\
   (1-q)p_1 &\leq qp_0 \\
    (1-q)(1-p_1) &\leq q(1-p_0) \\
    0 &\leq p_1 \leq p_0 \leq 1
\end{align*}
is $q = 1/2$ and $p_0 = p_1$, which sets both arguments of the $\max$ function equal. We conclude that in Case 1, the inequality 
$q \geq qp_0 + (1-q)(1-p_1)$.
must hold.

For Case 2, $\max\{q,1-q\} = 1-q$ and  $\min\{q,1-q\} = q$ and there is no solution to the set of inequalities
\begin{align*}
    1-q  &<  qp_0 + (1-q)(1-p_1) \\
    qp_0 &\leq (1-q)p_1\\
    q(1-p_0) &\leq (1-q)(1-p_1)\\
    0 &\leq p_1 \leq p_0 \leq 1.
\end{align*}
Hence, we can conclude that the inequality
$1-q \geq  qp_0 + (1-q)(1-p_1)$
must hold.

For Case 3, the accuracy is $qp_0 + (1-p_1)$. Combine this inequality with Lemma \ref{lem:beats_TTA} to obtain the proof for the first expression in the inequality. The second part of the inequality follows from  $qp_0 + (1-p_1)\geq \min\{q,1-q\}\cdot (p_0 + 1-p_1) = \min\{q,1-q\}\cdot(1+g)$. 
The fact that no other attack can perform better follows from Theorem \ref{thm:optimal}.
\end{proof}

\subsection{Comparison of BTTA to SMA}
Table \ref{table:compare_1} demonstrates how this most basic BTTA's expected precision compares to the experimentally observed precision of SMA attack in \cite{DBLP:conf/sp/ShokriSSS17}. %The precision values for  BTTA are computed from the formula given in Section \ref{sec:attack_prec_and_accuracy}. 
The value for $q$ used in the SMA experiments is $1/2$ and the values for $p_0$ and $p_1$ for each data set are shown in Table \ref{table:compare_1}. For every data set we fall into Case 3 and hence we use Equation \ref{eq:case3_prec} to compute precision (this is necessarily true when $q = 1/2$ and training accuracy exceeds testing).
What is remarkable is how often the BTTA, which is very simple, does comparably well to the much more complicated approach taken  wellin for SMA. The average difference in performance across all data sets tested is roughly $9\%$, but in cases where the SMA fails to obtain high precision the difference tends to be much smaller. In these cases, the result seems to imply that SMA is extracting little useful information about the behavior of the target model beyond the generalization gap between the training and testing accuracy of the  model.
%of the target model's accuracy on and off of the training set. 

In the next section we take the BTTA a step further, generalizing it to consider the training and testing accuracy on different portions of the partitioned data set. Armed with these new tools, we will return to our comparison with SMA.
\begin{table}
\small
\begin{tabular}{llllll}
 & \multicolumn{2}{c}{Accuracy} & \multicolumn{2}{c}{Precision} & \\
\midrule
Dataset                           & Train & Test & BTTA & SMA & Difference  \\
Adult                                & 0.848         & 0.842        & 0.502          & 0.503            & 0.00122 \\
MNIST                                & 0.984         & 0.928        & 0.515          & 0.517            & 0.00236 \\
Location                             & 1             & 0.673        & 0.598          & 0.678            & 0.0803 \\
Purchase(2)                         & 0.999         & 0.984        & 0.504          & 0.505            & 0.00122 \\
Purchase(10)                        & 0.999         & 0.866        & 0.536           & 0.550             & 0.0143 \\
Purchase(20)                        & 1             & 0.781        & 0.561          & 0.590             & 0.0285 \\
Purchase(50)                        & 1             & 0.693        & 0.591         & 0.860             & 0.269\\
Purchase(100)                       & 0.999         & 0.659        & 0.603          & 0.935            & 0.332 \\
TX hosp. stay                    & 0.668         & 0.517        & 0.564            & 0.657            & 0.0933 
\end{tabular}
\caption{Comparison of SMA experimental results to BTTA theoretical results.}
\label{table:compare_1}
\end{table}

\section{Expanding Bayesian Take-the-Typical Attack Beyond Generalization Gap}

It is natural to ask how much more effective a membership inference attack can be, given more information about the behavior of the model. Suppose for instance that one knows the accuracy of the model on the data points with true label $y$? Or instead, what if one knows the accuracy on the data points which are classified by the model as having label $y$? Or, even more complex, suppose rather than just a label, the model outputs a distribution on the different possible classifications? We would like to be able generalize our attack to all of these cases and more. The easiest way to do this is to reformulate the problem as knowing the information used in Bayesian Take-the-Typical Attack section, training accuracy $p_0$, testing accuracy $p_1$, and proportion in the training set $q$, but for different partitions of the data set $D$. We will first formalize this into an attack and then show how it applies to specific ways of partitioning the data. These formulations of the attack are the ones which we will use in our experiments.

%We generalize our earlier techniques by assuming the the data set $D$ can be partitioned into different sets where we have knowledge of training and testing accuracies on these sets. For example, consider knowledge of how well the model performs on data points with a particular true label, or even how well the model does on objects it classifies as having a certain label. The aim here to be as general as possible while still providing for an effective attack.
%We restate our attack more generally using Bayes rule as
%\begin{align*}
%P(x \in A \mid x, y, M(x) ) &= \frac{P(x, y, M(x) \mid x \in A)p(x\in A)}{P(x, y, M(x))}\\
%&= \frac{qP(x, y, M(x) \mid x \in A)}{P(x, y, M(x))}.
%\end{align*}
%Setting this greater or equal to $1/2$, we can our decision to report will be based on whether or not
%$$
%2qP(x,y,M(x) \mid x \in A) \geq P(x,y,M(x)).
%$$
%In the ideal scenario we would have a joint distribution on the random variables corresponding to $x$, $y$, and $M(x)$.

%In this section we explore some joint distributions an attacker may be able to glean easily from a target model, either through using the model as a black box, or training shadow to act as a proxy for the target model.

%\subsection{Categorical Generalization Gap}

\subsection{Categorical Bayesian Take-the-Typical Attack (CBTTA)}
\label{sec:categorical_BTTA}
Let $(D,P)$ be some (finite) data set of input-label pairs $z = (x,y)$ applied to target model $M$ with probability measure $P$. Partition $D$ into finitely many categories $(D_i)$ with sizes $(d_i)$ for (here, size refers to the ratio of selected data to the entire data set). Then partition each $D_i$ into a training set $D_i^{train}$ of size $d_i^{train}$ and a testing set $D_i^{test}$ of size $d_i^{test}$. Finally, suppose that model $M$ has accuracy $p_i^{train}$ on $D_i^{train}$ and accuracy $p_i^{test}$ on $D_i^{test}$ with $p_i^{train} \ge p_i^{test}$ for $i \le k$.

%In this paper, we investigate two settings: the first setting forms $C_i$ based on the true labels of the input data set $D$, and the second setting forms $C_i$ based on the perceived labels of the outputs of the target model $M$. 

As before, let $A := \{(x,y) \in D : M(x)=y\}$ denote the data on which $M$ is accurate. Then $d_i := P(D_i)$, $d_i^{train} := P(D_i^{train})$,  $d_i^{test} := P(D_i^{test})$, $p_i^{train} := P(A | D_i^{train})$, and $p_i^{test} := P(A | D_i^{test})$.

Note that categorical training proportion $q_i:=P(D_i^{train} | D_i)$, and categorical accuracy $p_i:=P(A | D_i)$ are defined in terms of previous parameters via $ q_i=d_i^{train} / d_i$ and $p_i =  [p_i^{train} d_i^{train} + p_i^{test} d_i^{test}] / d_i$. Furthermore, if $D^{train}:=\cup_{i=1}^k D_i^{train}$ and $D^{test}:=\cup_{i=1}^k D_i^{test}$, then the overall training proportion $q := P(D^{train})$ equals $\sum_{i=1}^k q_i d_i$ and the overall accuracy $p := P(A)$ is given by $p := \sum_{i=1}^k p_i d_i = \sum_{i=1}^k (p_i^{train} d_i^{train} + p_i^{test} d_i^{test})$. 
%(see Appendix 2Bmadelater).

We describe the Categorical Bayesian Take-the-Typical Attack (CBTTA) as follows: Assume that the attacker only has black box access to $M$, knows all data in category $D_i$ and knows parameter values $d_i^{train}$, $d_i^{test}$, $p_i^{train}$, and $p_i^{test}$ for all $i \le k$. Given data $z = (x,y)$, first determine the category $C_j$ which contains $z$; second, apply BTTA (as described above) with category train proportion  $q_j$, training accuracy $p_j^{train}$ and testing accuracy $p_j^{test}$.

\begin{theorem}
Given only access to target model $M$ (as a black box), knowledge of data in category $D_i$, and parameter values $d_i^{train}$, $d_i^{test}$, $p_i^{train}$, and $p_i^{test}$ for all $i \le k$, the Categorical Bayesian Take-the-Typical Attack is optimal with respect to accuracy.
\end{theorem}

\begin{proof}
If there was an attack that performed better overall on $D$ than CBTTA, then it would also perform better on $D_j$ than BTTA on $D_j$ for some $j \le k$ with only the parameters $q_j$, $p_j^{train}$, and $p_j^{test}$ (the other parameters are irrelevant). However, this contradicts Theorem \ref{thm:optimal}.
\end{proof}

In exchange for more effective attacks we pay the price in added complexity and the loss of nice explicit formulas for accuracy, precision, and recall. Next, we consider specific instances of CBTTA.

\textbf{Partition by True Label (PTL):}
In this attack the categories are defined by the true label, specifically $D_i = \{z = (x,y) \in D : y = i\}$. Assuming that the data set is partitioned into training and testing by sampling at random it is reasonable to assume that $q_i = q$, where $q$ is the overall proportion of data used for training. Obtaining the testing and training accuracies of target model will require a different idea, however. For this we use the idea of shadow models the same as \cite{DBLP:conf/sp/ShokriSSS17}. By training models which mimic the behavior of the target model we can observe in this process the individual training accuracy for each category. A similar tact will be adopted in our other attacks.

\textbf{Partition by Predicted Label (PPL):}
Now we consider the categories as defined by $D_i = \{z = (x,y) \in D : M(x) = i\}$. Like last time the training and texting accuracy can be observed from the shadow model. The only major difference arises from how we obtain the values for $q_i$. Now, we cannot assume $q_i = q$. Instead we use the final trained shadow models, apply them across $D$ and count for each category how many of the data points mapped to that category came from the training set.

\textbf{Partition by True Label Confidence (PTC):}
Like in partition by predicted label, the partitions are determined by the output of the model. 
In this attack the categories are defined by the confidence level assigned to the true label. The interval $[0,1]$ is partitioned into the subintervals $\mathcal{I}_1 = [0,1/n)$, $\mathcal{I}_2 = [1/n,2/n)$,..., $\mathcal{I}_n = [(n-1)/n, 1]$. Letting $M_y(x)$ denote the probability that the model assigns to $x$ having its true label $y$, the partitioning of $D$ is given by 
$D_i = \{z = (x,y) \in D : M_y(x) \in \mathcal{I}_i  \}$.
Testing and training accuracies along relative proportion in the training set can be obtained as above.

\textbf{Partition by Predicted Label Confidence (PPC):}
Here we look at the probabilities assigned (out of $m$ possible labels) to the model's predicted choice which lie in $[1/m, 1]$ (there is always a probability of at least $1/m$ in the prediction vector; otherwise they will not add to $1$) and partition the output into $n$ intervals of equal length $(m-1)/mn$; i.e., $\mathcal{I}_1 = [1/m, 1/m + (m-1)/mn), \mathcal{I}_2 = [1/m + (m-1)/mn, 1/m + 2(m-1)/mn), \hdots, \mathcal{I}_n = [1/m + (n-1)(m-1)/mn, 1]$. Letting $M_j(x)$ denote the probability that the model assigns to $x$ having label $j$, the partitioning is given by $D_i := \{z = (x,y) \in D : \max_j (M_j(x)) \in \mathcal{I}_i \}$ for $i \le n$. The necessary values for training accuracy, testing accuracy, and $q_i$ can all be approximated through shadow models.

\textbf{Combining Partition Techniques:}

The techniques above can be combined to increase the effectiveness of the CBTTA. To do so, you intersect categories from each method to form the categories for the combination. 
Some examples of combinations like Predicted Label with Predicted Label Confidence or True Label with True Label Confidence are reasonable and intuitive to employ because a model may perform differently on data between different confidence levels and between labels. However, it is ill-advised to combine the True Label and the Predicted Label techniques, as it results in trival categories with either zero accuracy or perfect accuracy, which leads the attack becoming TTA.

This list is by no means exhaustive. One interesting consequence of the optimality of CBTTA and the possibility of trying endless combinations of ways to partition the data set is the idea of using the partitions to uncover which information other attacks are exploiting to successfully perform membership inference. We will see that the way the data is partitioned can have a noticeable effect on the performance of CBTTA.

\section{Experimental Evaluation}
Here we present the results from implementations of our Bayesian attack and the state-of-the-art shadow model membership inference attack. All of our experiments attack a target neural network trained on the CIFAR-10 dataset. In order to produce target networks with generalization gaps of different magnitudes, we vary the size of the training set. 

\subsection{Experimental Setup}
\textbf{Data:} We use CIFAR-10, a benchmark dataset for image recognition. CIFAR-10 is made of $60,000$ $32\times32$ color images in $10$ different classes, with $6,000$ images in each class. We train our target model on training sets of sizes $2500, 5000, 10000,$ and $15000$ images from CIFAR-10.
\\
\textbf{Target Model:} Our target model is a convolutional neural network with two convolutional and max pooling layers, two hidden layers of sizes $120$ and $84$, and a SoftMax output. Our activation function is ReLu. We chose our learning rate to be $0.001$ and our maximum number of epochs of training to be $100$.
\\
\textbf{Shadow Model Attack:} We replicated the shadow model attack presented in \cite{DBLP:conf/sp/ShokriSSS17}. For each target model, we trained $10$ shadow models on CIFAR-10 data disjoint from the target training set. These shadow models had the same architecture as the target model and were trained identically. We then collected the outputs of the shadow models on training and testing data to train the attack neural networks. For each category of images in CIFAR-10, we trained an attack neural network to predict whether or not a given image was in the shadow training set. Our attack networks had two hidden layers of size 50 and used ReLu activation functions.

\subsection{Results}
The objective of each attacker is to determine which data points were in the target model's training set. 
We evaluate our attacks and compare the results to the Shadow Model Attack (SMA) by executing these procedures on random samples of the target model's train and test data sets. In our evaluation we use train and test data sets of identical sizes, so that the baseline membership inference attack accuracy for random guessing is 0.5.

\begin{table}
\centering
\begin{tabular}{llll}
                \textbf{Attack}
               & \textbf{Accuracy} & \textbf{Precision} & \textbf{Recall} \\ \hline
\textbf{BTTA}  &        0.746           &       0.663             &         1.000        \\ 
\textbf{PTL}  &        0.746       &         0.663           &        1.000         \\ 
\textbf{PPL} & 0.746 & 0.663 & 1.000 \\
\textbf{PTC} &        0.777           &      0.725              &      0.988      \\ 
\textbf{PPC} &        0.778           &      0.728              &      0.985      \\ 
\textbf{SMA}   &          0.793         &       0.715             &         0.973        \\ 
\end{tabular}
\caption{Performance of Bayesian Attacks and Shadow Model Attack (SMA) on the target model trained on 10,000 CIFAR-10 data points. Baseline accuracy 0.5 (See Categorical Bayesian Take-the-Typical Attack section) for a description of each attack).}
\label{table:attack_table}
\end{table}

\begin{figure}[ht]
\centering
\includegraphics{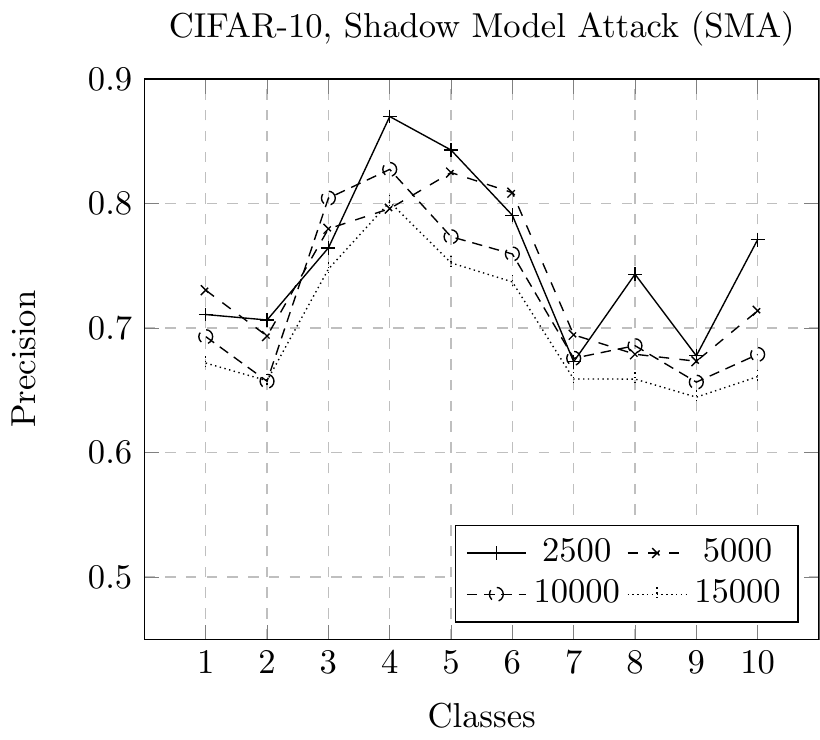}
\includegraphics{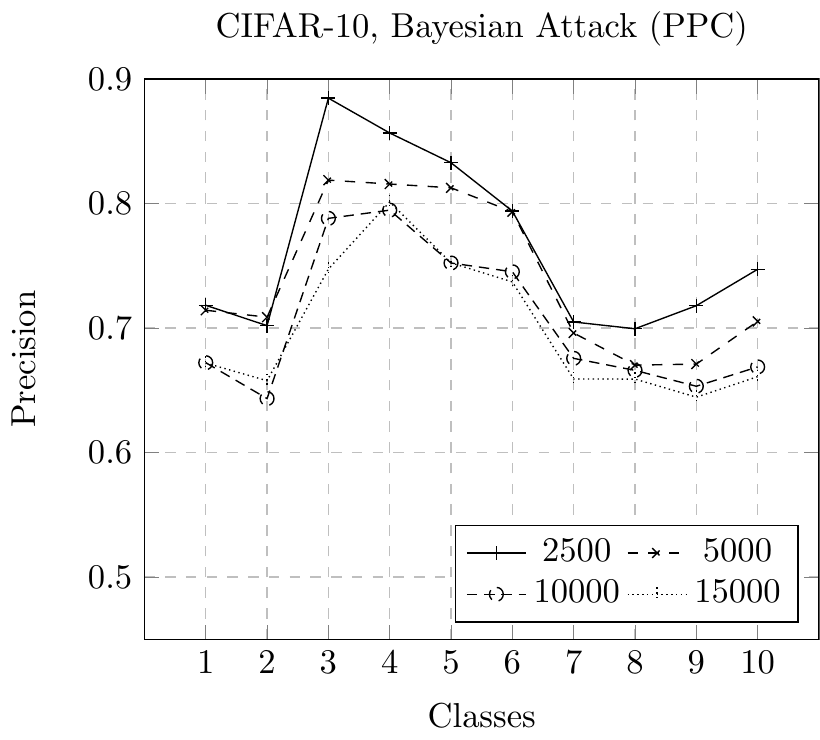}
\caption{Precision of SMA and PPC on CIFAR-10.}
\label{graph:line_graph}
\end{figure}

\begin{figure}[ht]
\centering
\includegraphics{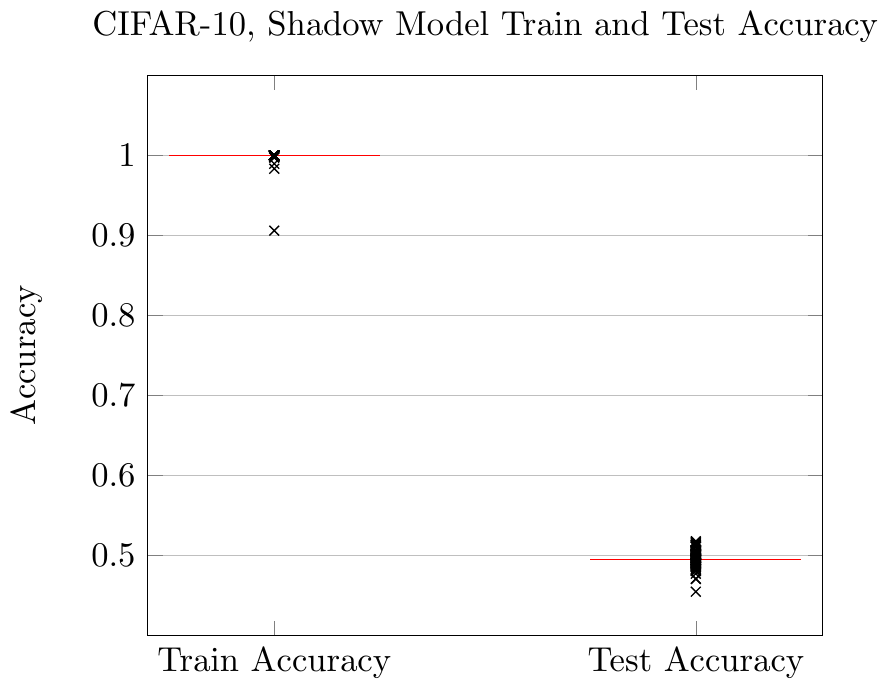}
\caption{Scatter plot showing the training and testing accuracies of the shadow models trained on 10,000 CIFAR-10 images. The training and testing accuracy of the corresponding target model is indicated by the red line.}
\label{figure:train_test_acc}
\end{figure}

% \begin{figure}
% \begin{tikzpicture}
% \begin{axis}[
%     title={CIFAR-10, Shadow Model Train and Test Accuracy},
%     ylabel={Accuracy},
% xmin=0.7, xmax = 2.3,
% ymin=0.4,ymax=1.1,xtick={1,2},xticklabels={Train Accuracy, Test Accuracy},
% ytick={0,0.5, 0.6, 0.7, 0.8, 0.9, 1.0},ymajorgrids]
% \addplot+[only marks,mark=x,color=black] 
%  table {shadow_train.dat};
% \addplot[only marks,mark=x,color=black] 
%  table {shadow_test.dat};
% \addplot[ color=red
%     ]     coordinates {
% (1.75, 0.495)
% (2.25, 0.495)
%     };
% \addplot[ color=red
%     ]     coordinates {
% (0.75, 1.0)
% (1.25, 1.0)
%     };
    
% \end{axis}
% \end{tikzpicture}
% \caption{Scatter plot showing the training and testing accuracies of the shadow models trained on 10,000 CIFAR-10 images. The training and testing accuracy of the corresponding target model is indicated by the red line.}
% \label{figure:train_test_acc}
% \end{figure}

In general, we found the performance of the categorical attacks, and particularly the predicted label partition attack, comparable with the performance of the shadow model attack. This is striking considering that our attack has no parameters or training procedure and only relies on a handful of statistics about the target model.

Interestingly, the different partition schemes in our categorical Bayesian attacks provide only a moderate improvement to the vanilla Bayesian Take the Typical Attack (BTTA). Indeed, in Table \ref{table:attack_table} the Partition by True Label (PTL) and Partition by Predicted Label (PPL) attacks have performance identical to BTTA. We note that if all categories of a partitioned Bayesian attack CA1 are a subset of some category in a partitioned Bayesian attack CA2, then the accuracy of CA1 is no less than that of CA2.
 
For all remaining evaluations of our Bayesian attacks, we use the predicted label partitioning scheme (PPL). We plot the categorical precision of the shadow model attack and the Bayesian attack on all CIFAR-10 data sets in  Figure \ref{graph:line_graph}. We found the recall to be above $0.95$ for all attackers on all data sets.

Not surprisingly, the  accuracies of the Bayesian attack and the Shadow Model Attack of \cite{DBLP:conf/sp/ShokriSSS17} decrease as the target model is trained on larger data sets. As the train set size increases, the target model becomes less overfit, decreasing its generalization gap and its vulnerability to attack. The only information exploited by our Bayesian attack is the categorical generalization gap. Because the Shadow Model Attack performs comparably to the Bayesian attack, it must be extracting the same information as the Bayes attack or information with comparable predictive power. In the
%Figure \ref{figure:train_test_acc} 
CIFAR-10, Shadow Model Train and Test Accuracy Figure (see Appendix) it can be seen that the train and test accuracies of the $100$ shadow models are distributed tightly near the training and testing accuracies of the target model. This suggests that the shadow data the attack networks trained on has the same train and test statistics as the target model, and that this generalization gap is what's being learned by the attack models of \cite{DBLP:conf/sp/ShokriSSS17} during training.

It is worth discussing the different assumptions made by each attack. The Shadow Model Attack assumes we have access to data drawn from the same distribution as the target model's training data, as well as the target model's architecture and training procedure. On the other hand, our Bayesian attack assumes only knowledge of some true statistics of the target model's performance on the train and test sets. This assumption is reasonable in practice as the categorical training and testing accuracy of a machine learning model is frequently made public. 

% \begin{figure}
% \begin{tikzpicture}
% \begin{axis}[
%     title={CIFAR-10, Shadow Model Train and Test Accuracy},
%     ylabel={Accuracy},
% xmin=0.7, xmax = 2.3,
% ymin=0.4,ymax=1.1,xtick={1,2},xticklabels={Train Accuracy, Test Accuracy},
% ytick={0,0.5, 0.6, 0.7, 0.8, 0.9, 1.0},ymajorgrids]
% \addplot+[only marks,mark=x,color=black] 
%  table {shadow_train.dat};
% \addplot[only marks,mark=x,color=black] 
%  table {shadow_test.dat};
% \addplot[ color=red
%     ]     coordinates {
% (1.75, 0.495)
% (2.25, 0.495)
%     };
% \addplot[ color=red
%     ]     coordinates {
% (0.75, 1.0)
% (1.25, 1.0)
%     };
    
% \end{axis}
% \end{tikzpicture}
% \caption{Scatter plot showing the training and testing accuracies of the shadow models trained on 10,000 CIFAR-10 images. The training and testing accuracy of the corresponding target model is indicated by the red line.}
% \label{figure:train_test_acc}
% \end{figure}

\section{Conclusion}
We introduced the Bayesian Take-the-Typical Attack (BTTA), a simple, yet effective, attack. 
The BTTA algorithm requires no trained attack model, unlike many other MIA methods, including the Shadow Model Attack of \cite{DBLP:conf/sp/ShokriSSS17}. 
%This attack Perhaps, of as much value as its practical usage however is that it allows one to access the risk of membership inference, provided they have knowledge of what the attacker knows.
Furthermore, we showed that in restricted settings where the attacker knows only the testing accuracy, training accuracy, the knowledge of the proportion of the training set to the total data, along with the label predicted by the model for all data points, BTTA is a provably optimal attack. 

Additionally, we generalized BTTA to a more sophisticated class of attacks, CBTTA, where the attacker has more knowledge of the model. We again proved this attack is theoretically optimal in terms of expected accuracy, and performed experimental comparisons with a state-of-the-art MIA method. In doing so we were able to observe in which cases the state-of-the-art attack learned more useful information than our CBTTA attacker had access to.

\bibliography{ref}

\newpage
\section*{Appendix A}

{\bf Assumptions Known by Attacker.}\\
For all $i \le k$:
\begin{align*}
D_i& & & \text{i-th category}\\
d_i^{train} &:= P(D_i^{train}) & & \text{size of data in } D_i^{train}\\
d_i^{test} &:= P(D_i^{test}) & & \text{size of data in } D_i^{test}\\
p_i^{train} &:= P(A | D_i^{train}) & & \text{accuracy of } M \text{ in } D_i^{train}\\
p_i^{test} &:= P(A | D_i^{test}) & & \text{accuracy of } M \text{ in } D_i^{test}
\end{align*}

{\bf Other Features.}\\
For all $i \le k$:
\begin{align*}
d_i &:= P(D_i) & & \text{size of data in } D_i \\
q_i &:= P(D_i^{train} | D_i) & & \text{training proportion in } D_i\\
p_i &:= P(A | D_i) & & \text{accuracy of } M \text{ in } D_i \\
q &:= P(D^{train}) & & \text{overall training proportion}\\
p &:= P(A) & & \text{overall accuracy of } M \\
\end{align*}
We provide details for various statements and claims from Section on the Categorical Bayesian Take-the-Typical Attack. %\ref{sec:categorical_BTTA}. 

{\bf Claim 1.} $$d_i =  d_i^{train} + d_i^{test}$$
\begin{align*}
d_i :&= P(D_i) = P(D_i^{train} \cup D_i^{test})\\
&= P(D_i^{train}) + P(D_i^{test}) = d_i^{train} + d_i^{test}. \end{align*}

{\bf Claim 2.} $$q_i d_i = d_i^{train}$$
\begin{align*}
q_i d_i :&= P(D_i^{train}|D_i) P(D_i) =  P(D_i^{train} \cap D_i)\\ 
&= P(D_i^{train}) = d_i^{train}. \end{align*}

{\bf Claim 3.} $$p_i d_i = p_i^{train} d_i^{train} + p_i^{test} d_i^{test}$$
\begin{align*}
p_i d_i :&= P(A|D_i) P(D_i) =  P(A \cap D_i)\\
&= P(A \cap D_i^{train}) + P(A \cap D_i^{test})\\
&= P(A | D_i^{train}) P(D_i^{train}) + P(A | D_i^{test}) P(D_i^{test}) \\
&= p_i^{train} d_i^{train} + p_i^{test} d_i^{test}. \end{align*}

{\bf Claim 4.} $$q =  \sum_{i=1}^k q_i d_i$$
\begin{align*}
q :&= P(D^{train}) = \sum_{i=1}^k P(D_i^{train})  \\
&= \sum_{i=1}^k d_i^{train} = \sum_{i=1}^k q_i d_i.
\end{align*}

{\bf Claim 5.} $$p =  \sum_{i=1}^k p_i d_i$$
\begin{align*}
p :&= P(A) = \sum_{i=1}^k P(A \cap D_i) = \sum_{i=1}^k p_i d_i. \end{align*}

\end{document}